\documentclass[reqno]{amsart}
\usepackage{lineno}
\usepackage{xcolor}
\usepackage{graphicx}
\theoremstyle{plain}
\begingroup
 
\newtheorem{theorem}{Theorem} 
\newtheorem{corollary}{Corollary}
\newtheorem{lemma}{Lemma}
\endgroup
\newcommand{\nwc}{\newcommand}
\nwc{\qref}[1]{(\ref{#1})}
\nwc{\cadlag}{c\`{a}dl\`{a}g}
\nwc{\la}{\label}
\nwc{\nn}{\nonumber}
\nwc{\Z}{\mathbb{Z}}
\nwc{\C}{\mathbb{C}}
\nwc{\T}{\mathbb{T}}
\nwc{\E}{\mathbb{E}}
\nwc{\R}{\mathbb{R}}
\nwc{\N}{\mathbb{N}}
\nwc{\Rn}{\mathbb{R}^n}
\nwc{\PP}{\mathcal{P}}
\nwc{\M}{\mathcal{M}}
\nwc{\Ito}{It\^{o}}
\nwc{\DiffM}{\mathrm{Diff}(M)}
\nwc{\orbit}{\mathcal{O}}
\nwc{\bbb}{\mathbf{M}}

\nwc{\law}{\stackrel{\mathcal{L}}{\rightarrow}}
\nwc{\eqd}{\stackrel{\mathcal{L}}{=}}
\nwc{\vp}{\varphi}
\nwc{\veps}{\varepsilon}
\nwc{\eps}{\veps}
\nwc{\dnto}{\downarrow}
\nwc{\nsup}{^{(n)}}
\nwc{\ksup}{^{(k)}}
\nwc{\jsup}{^{(j)}}
\nwc{\nksup}{^{(n_k)}}
\nwc{\inv}{^{-1}}
\nwc{\argmin}{\mathrm{argmin}}
\nwc{\argmax}{\mathrm{argmax}}
\nwc{\Tr}{\mathrm{Tr}}
\nwc{\Id}{\mathrm{Id}}
\nwc{\Pn}{\mathbb{P}(n)}
\nwc{\PnR}{\mathbb{P}(n;\mathbb{R})}
\nwc{\PnC}{\mathbb{P}(n;\mathbb{C})}
\nwc{\Hn}{\mathbb{H}(n)}
\nwc{\HnR}{\mathbb{H}(n;\mathbb{R})}
\nwc{\HnC}{\mathbb{H}(n;\mathbb{C})}
\nwc{\An}{\mathbb{A}(n)}
\nwc{\AnR}{\mathbb{A}(n;\mathbb{R})}
\nwc{\AnC}{\mathbb{A}(n;\mathbb{C})}
\nwc{\Mn}{\mathbb{M}(n)}
\nwc{\Md}{\mathbb{M}_d}
\nwc{\Mfd}{\mathfrak{M}_d}
\nwc{\Mm}{\mathbb{M}_m}
\nwc{\Mfm}{\mathfrak{M}_m}
\nwc{\feasible}{\mathcal{F}}
\nwc{\GLn}{GL(n)}
\nwc{\GLd}{GL(d)}
\nwc{\gld}{gl(d)}
\nwc{\GLnC}{GL(n;\mathbb{C})}
\nwc{\GLnR}{GL(n;\mathbb{R})}
\nwc{\GLdC}{GL(d;\mathbb{C})}
\nwc{\gldC}{gl(d;\mathbb{C})}
\nwc{\GLdR}{GL(d;\mathbb{R})}
\nwc{\Poincare}{Poincar\'{e}}
\nwc{\Un}{U(n)}
\nwc{\Sn}{\mathbb{S}^{n}}
\nwc{\Tn}{\mathbb{T}^n}
\nwc{\ev}{\mathrm{ev}}
\nwc{\Diff}{\mathrm{Diff}}
\nwc{\orbitx}{\mathcal{O}_X}
\nwc{\orbitxG}{\mathcal{O}_{X,\mathbf{G}}}
\nwc{\volume}{\mathrm{vol}}
\nwc{\symm}{\mathrm{Symm}}
\nwc{\psd}{\mathbb{P}}
\nwc{\od}{O_d}
\nwc{\balance}{\mathcal{M}}
\nwc{\balanceG}{\mathcal{M}_{\mathbf{G}}}
\nwc{\nstar}{n_*}
\nwc{\manmap}{\mathfrak{x}}
\nwc{\manmapbeta}{\mathfrak{y}}
\nwc{\grad}{\mathrm{grad}}
\nwc{\ww}{\mathbf{W}}
\nwc{\ff}{\mathcal{F}}
\nwc{\ggg}{\mathcal{G}}
\nwc{\hh}{\mathcal{H}}
\nwc{\ts}{n}
\nwc{\anti}{\mathbb{A}}
\nwc{\diag}{\mathrm{diag}}
\nwc{\Lojas}{Lojasiewicz}
\nwc{\bd}{\mathbf}

\nwc{\Mr}{\mathfrak{M}_r}
\nwc{\Mdd}{\mathfrak{M}_d}
\nwc{\bfW}{\mathbf{W}}
\nwc{\bfG}{\mathbf{G}}
\nwc{\bfw}{\mathbf{w}}
\nwc{\bfA}{\mathbf{A}}
\nwc{\bfB}{\mathbf{B}}
\nwc{\bfa}{\mathbf{a}}
\nwc{\bfb}{\mathbf{b}}
\nwc{\bfaa}{\mathbf{alpha}}
\nwc{\bfQ}{\mathbf{Q}}
\nwc{\bfC}{\mathbf{C}}
\nwc{\bfc}{\mathbf{c}}
\nwc{\bfl}{\mathbf{l}}
\nwc{\bfm}{\mathbf{m}}
\nwc{\bfu}{\mathbf{u}}
\nwc{\bfv}{\mathbf{v}}
\nwc{\gbw}{g^{BW}}
\nwc{\fiber}{\mathcal{F}}
\nwc{\fiberX}{\mathcal{F}_X}
\nwc{\aaa}{\mathcal{A}}
\nwc{\Stief}{\mathrm{St}}
\nwc{\Fr}{\iota}
\nwc{\weyl}{\mathcal{W}}
\nwc{\Her}{\mathrm{Her}}
\nwc{\refspace}{\mathcal{M}}
\nwc{\refgroup}{\mathcal{G}}
\nwc{\quotientspace}{\mathcal{X}}
\nwc{\refmetric}{g}
\nwc{\quotientmetric}{h}
\nwc{\moment}{\mu}
\nwc{\MdC}{\mathbb{M}_d(\mathbb{C})}
\nwc{\regp}{R_p}
\nwc{\Ggroup}{\mathcal{G}}
\nwc{\Galgebra}{\mathfrak{g}}
\nwc{\Kgroup}{\mathcal{K}}
\nwc{\Hgroup}{\mathcal{H}}
\nwc{\Kalgebra}{\mathfrak{k}}
\nwc{\groupx}{\mathcal{G}_x}
\nwc{\stabx}{\mathcal{H}_x}

\theoremstyle{definition}
\newtheorem{defn}[theorem]{Definition} 
\newtheorem{remark}[theorem]{Remark}

\theoremstyle{remark}
 
\numberwithin{equation}{section}
\numberwithin{figure}{section}

\begin{document}


\title{Regularization implies balancedness in the Deep Linear Network}

\begin{abstract}
We use geometric invariant theory (GIT) to study the deep linear network (DLN). The Kempf-Ness theorem is used to establish that the $L^2$ regularizer is minimized on the balanced manifold. We introduce related balancing flows using the Riemannian geometry of fibers. The balancing flow defined by the  $L^2$ regularizer is shown to converge to the balanced manifold at a uniform exponential rate. The balancing flow defined by the squared moment map is computed explicitly and shown to converge globally.

This framework allows us to decompose the training dynamics into two distinct gradient flows: a  regularizing flow on fibers and a learning flow on the balanced manifold. It also provides a common mathematical framework for balancedness in deep learning and linear systems theory. We use this framework to interpret balancedness in terms of fast-slow systems, model reduction and Bayesian principles. 

\end{abstract}

\author{Kathryn Lindsey}
\address{Department of Mathematics, Maloney Hall, Boston College, Chestnut Hill, MA 02467-3806}
\email{lindseka@bc.edu}

\author{Govind Menon}
\address{Division of Applied Mathematics, Brown University, 182 George St., Providence, RI 02912.}
\email{govind\_menon@brown.edu}
\curraddr{School of Mathematics, Institute for Advanced Study, 1 Einstein Drive, Princeton, NJ 08540}
\email{gmenon@ias.edu}

\thanks{KL was supported by NSF DMS \#2133822. GM 
was supported by the NSF grants DMS \#2107205 and DMS \#2407055 and the Erik Ellentuck Fellow Fund at the Institute for Advanced Study.}


\maketitle

{\centering \em For David Mumford. \par}

\maketitle

\section{Overview}
\subsection{The main results}
This paper is the second of a series on the mathematical structure of the Deep Linear Network (DLN). Our goal in these works is to shed new light on the training dynamics on deep learning through an analysis of the hidden symmetries of the DLN. This paper focuses on the concept of balancedness; we observe the natural role of Geometric Invariant Theory (GIT) in the DLN and use the Kempf-Ness theorems and related gradient flows to establish a new dynamic paradigm in deep learning. 

We summarize the essentials of the DLN below and refer to~\cite{MY-dln} for further context. The state space of the DLN consists of $N$ $d\times d$ matrices $\ww=(W_N,\ldots,W_1)\in \Md^N(\mathbb{F})$. The integers $N$ and $d$ are termed the depth and width respectively. The field $\mathbb{F}$ may be either $\mathbb{R}$ or $\mathbb{C}$ and we use $^*$ to denote the transpose or conjugate transpose respectively. The cost function $E$ for the DLN depends only on the end-to-end matrix 
\begin{equation}
    \label{eq:balanced-intro1}
    X= W_N W_{N-1} \ldots W_1.
\end{equation}
Training dynamics in the DLN is modeled by the gradient flow 
\begin{equation}
    \label{eq:balanced-intro1a}
    \dot{\mathbf{W}} = - \nabla_\mathbf{W}E(X(\mathbf{W})).
\end{equation}
The underlying Euclidean metric is described in Section~\ref{subsec:reg-flow} below.


The natural geometry of equation~\eqref{eq:balanced-intro1a} is that of a gradient flow on a fibered space. It is this aspect that we study in detail. Our results in this paper rely on the interplay between two complementary foliations of $\Md^N$: 
\begin{enumerate}
    \item The 
fibers $\fiberX$ of $\Md^N$ defined as the solution set to the polynomial equation~\eqref{eq:balanced-intro1} for each $X \in \Md$. 
\item the balanced variety 
$\balance_0$ defined as the solution set to the quadratic system 
\begin{equation}
    \label{eq:balanced-intro2}
     W_kW_k^* = W_{k+1}^*W_{k+1}, \quad 1\leq k \leq N-1.
\end{equation}
\end{enumerate}
The balanced variety is invariant under the gradient flow~\eqref{eq:balanced-intro1a} and has several fundamental properties (see~\cite[Thms. 2-8]{MY-dln})). It is foliated by rank into a collection of manifolds. When $X$ has full rank, it lies on a leaf of $\balance_0$, termed the balanced manifold $\balance$. We assume throughout this paper that $X$ has full rank in order to illustrate the new ideas without technical complications. The fibers and balanced manifold are illustrated schematically in Figure~\ref{fig:dln}.


We establish a minimum principle for balancedness that reveals a `hidden convexity' in deep learning. By hidden convexity we mean that the balanced manifold can be characterized by a class of minimum principles of which the following is the simplest. Consider the $L^2$ (ridge) regularizer 
\begin{equation}
    \label{eq:ridge}
    \|\ww\|_2^2 := \sum_{k=1}^N \Tr (W_k^* W_k).
\end{equation}

\begin{theorem}
\label{thm:intro} Assume $X$ has full rank. Then
\begin{equation}
\label{eq:variation1} \argmin_{\ww \in \fiberX} \|\ww\|_2 = \fiberX \cap \balance.
\end{equation}   
\end{theorem}
In what follows, we provide several interpretations and extensions of this theorem. The most important in practice is model reduction: Theorem~\ref{thm:intro} selects the simplest parametric description of a given end-to-end matrix $X$ (Section~\ref{subsec:occam}).  Our ideas on balancing are closely tied to linear systems theory; we show that both areas may be unified through the use of GIT (Sections~\ref{subsec:ls}--\ref{subsec:GIT}). In order to implement this idea in practice, we introduce a broader class of minimization principles along with new gradient flows ({\em regularizing flows\/}) with rigorous convergence guarantees (Section~\ref{subsec:reg-flow}). Finally, we use GIT to relate the DLN to mathematical physics and the Yang-Mills theory. Let us now explain these ideas in greater detail.

\subsection{Balancedness, regularization and Occam's razor}
\label{subsec:occam}
Theorem~\ref{thm:intro} is a form of Occam's razor. This is seen as follows.

Let $X = Q_N\Sigma Q_0^*$ denote the SVD of $X$. We define the {\em center\/} of $\fiberX$ to be the point $\bfC= (Q_N\Lambda,\ldots,\Lambda Q_0^*)$, with $\Lambda=\Sigma^{1/N}$.  Every point on $\fiberX$ may be obtained by translating the center through a group action. 

Given $N-1$ invertible matrices, $\mathbf{A}=(A_{N-1},A_{N-2},\ldots, A_1)$, we define the $\GLdC^{N-1}$ action 
\begin{equation}
    \label{eq:group-action1}
    \mathbf{A}\cdot \ww = (W_N A_{N-1}^{-1}, A_{N-1}W_{N-1}A_{N-2}^{-1}, \cdots, A_1 W_1).
\end{equation}               
This group action leaves $\fiberX$ invariant. We show (Lemma~\ref{le:group-orbit}) that each point in $\fiberX$ is of the form $\mathbf{A}\cdot \bfC$ for some $\mathbf{A} \in \GLdC^{N-1}$. 

Similarly, $\fiberX \cap \balance:=\orbitx$  is a $U_d^{N-1}$ group orbit, where $U_d$ is the unitary group. Each $\ww \in\orbitx$ is obtained by the group action $\bfQ\cdot \bfC$ where $\bfQ=(Q_{N-1},\ldots,Q_1)\in U_d^{N-1}$ (this is an easy modification of~\cite[\S 4]{MY-dln}).

The unitary orbit $\orbitx$ consists of the simplest parametric representations of $X$ amongst all admissible parametrizations $\ww \in \fiberX$. Certainly $\bfC= (Q_N \Lambda,\ldots,\Lambda Q_0^*)$ is a point in $\fiberX$ that contains no superfluous information: it depends on $X$ and $X$ alone. Further, since  
\begin{equation}
    \label{eq:group-action2}
    \|\ww\|_2 = \|\bfQ \cdot \ww\|_2, \quad \bfQ \in U_d^{N-1},
\end{equation} 
the minimizing set of $\|\ww\|_2$ must be invariant under  the $U_d^{N-1}$ action. 

Thus, Theorem~\ref{thm:intro} tells us that minimizing the $L^2$ regularizer, conditional on the end-to-end matrix $X$, yields the simplest parametric representations of $X$. It is in this sense that regularization in the DLN acts as a form of Occam's razor. 

\subsection{Balancedness in deep learning and linear systems theory}
\label{subsec:ls}
The concept of balancedness has arisen independently in linear systems theory and deep learning. Our main insight is that these concepts may be unified, allowing us to transport techniques used in linear systems theory to the DLN. In particular, we follow the work of Helmke to prove Theorem~\ref{thm:intro}~\cite{Helmke}. We augment his work with an analysis of the Riemannian geometry of $\fiberX$ and a rigorous analysis of balancing flows, providing an independent proof of Theorem~\ref{thm:intro}.

\subsubsection{Linear systems theory\/} 
The concept of balancedness arises in linear systems theory as follows. We consider the linear system
\begin{equation}
    \label{eq:ls1}
    \dot{x} = A x + B u, \quad y = Cx,
\end{equation}
where $x \in \C^n$ is the state, $u \in \C^m$ is the control, $y\in \C^p$ is the observation, and $A$, $B$ and $C$ are time-independent matrices with the appropriate dimensions. The input-output relation for this system is a relationship between the functions $u(t)$ and $y(t)$, $t \in [0,\infty)$, mediated by the equation~\eqref{eq:ls1}. It may be studied in the frequency domain through the Hankel matrix 
\begin{equation}
    \label{eq:ls2} H(z) = C (zI - A)^{-1}B, \quad z \in \C.
\end{equation}
Since $H(z)$ is unchanged under the action
\begin{equation}
    \label{eq:ls3}
    (A,B,C) \mapsto (MAM^{-1}, MB, CM^{-1}), \quad M \in \GLn,
\end{equation}
the class of triples ${(A,B,C)}$ whose Hankel matrix is $H(z)$ is a $\GLnC$ orbit.
Each choice $(A,B,C)$ that satisfies~\eqref{eq:ls2} is a realization of a linear system (the model) that is consistent with the input-output relation (the data). This notion dates to the work of Kalman~\cite{Kalman-real}. 

Model reduction in this context is the choice of an optimal realization consistent with the data. Norm balanced realizations minimize the Frobenius norm $\|MAM^{-1}\|_2^2 + \|MB\|_2^2 + \|CM^{-1}\|_2^2$ over $M \in \GLnC$. They constitute a principled choice of an optimal realization and have several favorable properties~\cite{Helmke}. This concept continues to be of great importance in Structured State Space Model (SSM) architectures, which are principled alternative to transformer architectures for modeling long sequences~\cite{Gu1}. In particular, the analysis of implicit bias in SSMs in analogy with implicit bias in the DLN is of current importance~\cite{Cohen-Karlik1,Cohen-Karlik2}.

\subsubsection{Balancedness in deep learning.\/} 
The concept of balancedness for the DLN was introduced by Arora, Cohen and Hazan in~\cite{ACH}. The underlying heuristic that `load is equally distributed across a balanced network' was formalized by Du, Hu and Lee for fully connected feed-forward networks with a homogeneous nonlinearity~\cite[Theorem 2.1]{JasonLee}. In our notation, this is the observation that when $\ww \in \balance$, then $\|W_k\|_2^2$ is independent of $k$. However, for the DLN, more is true. The singular values and singular vectors are aligned across the network: the SVD of $W_k = U_k \Lambda_k V_k^*$ and $W_{k+1}$ are related through $\Lambda_k=\Sigma^{1/N}$ for all $k$ and $V_k=U_{k+1}$ for $1\leq k \leq N-1$. This notion of alignment was examined by Ji and Telgarsky in several instances~\cite{Telgarsky}. The relationship between balancing and regularization appears also in the work of Soltanolkotabi, St{\"o}ger, and Xie (for $N=2$ and $W_2=W_1^*)$~\cite{SSX}. 

The surprising appearance of the conservation laws for the DLN (the moments $\bfG$ defined in equation~\eqref{eq:def-G} below) has also attracted attention. In several recent papers, Marcotte, Gribonval and Peyr\'{e}\/ have studied the relation between the symmetries and conservation laws for various neural networks~\cite{Peyre1,Peyre2,Peyre3}. Minimum principles for balancing weights, including an algorithm for balancing, have been introduced by Saul~\cite{Saul}. This work draws connections between symmetry in deep learning and mathematical physics in a manner that is similar in spirit to our work. Several other recent works have investigated equivariance and symmetry in deep learning~\cite{Kunin1,Tanaka-Kunin,Walters}. Finally, we note that the interplay between minimum principles and flatness has been studied by Ding, Drusvyatskiy, Fazel and Harchaoui~\cite{DDFH} and that a stratification of the loss landscape for the quadratic loss function has been studied by Achour, Malgouyres and Gerchinovitz~\cite{Achour}.

Our work does not rely on the techniques in the above papers though it builds on these themes. Instead, our work is closest in spirit to the relationship between optimization problems and integrable gradient flows on adjoint orbits studied by Bloch, Brockett and collaborators in the 1990s~\cite{Bloch1994,BBR}. Recent developments in this vein have extended these geometric principles to high-dimensional machine learning and quantum control, reinforcing the notion that the training dynamics of deep architectures are fundamentally governed by the underlying Riemannian geometry and conservation laws induced by parameter space symmetries~\cite{wang2024dynamical}.


\subsection{GIT, the moment map and duality}
\label{subsec:GIT}
The main goal in Geometric Invariant Theory (GIT) is to classify the orbit space of a group acting on a vector space~\cite{Mumford-Fogarty}. We do not study the orbit space of the DLN in full generality. However, we note some important consequences of the general theory.


\subsubsection{The moment map}
Define the $N-1$ Hermitian matrices
\begin{equation}
    \label{eq:def-G} G_k = W_kW_k^* - W_{k+1}^*W_{k+1}, \quad 1\leq k \leq N-1.
\end{equation}
The Hermitian matrices $\{G_k\}_{k=1}^N$ are conserved under the gradient flow of an arbitrary loss function $E$~\cite[Theorem 2]{MY-dln}. The appearance of such a large number of conservation laws for a {\em gradient\/} flow is surprising at first sight. Indeed, we expect conservation laws for {\em Hamiltonian\/} systems, not gradient flows!   

The Kempf-Ness theorem explains this phenomena. The $\{G_k\}_{k=1}^{N-1}$ are obtained from the moment map corresponding to the invariance of $\tfrac{1}{2}\|\ww\|_2^2$ under the action
\begin{equation}
    \label{eq:moment1} \ww \mapsto \mathbf{U}\cdot \ww : = (W_N U_{N-1}^*,U_{N-1} W_{N-1} U_{N-2}^*,\cdots, U_1 W_1),
\end{equation} 
for $\mathbf{U}=(U_{N-1},\ldots, U_1) \in U_d^{N-1}$. The moment map is simply
\begin{equation}
    \label{eq:moment2}
    \MdC^N \to \Her_d^{N-1}, \quad \ww \mapsto \mathbf{G}:=(G_{N-1},\cdots,G_1).
\end{equation}
We explain how these arise in Section~\ref{sec:kempf-ness}. Loosely speaking, the moments $\mathbf{G}\in \Her_d^{N-1}$ may be seen as the analogues in deep learning of dual variables in the study of conic programs. 


\subsubsection{Summary}
In terms of mathematical structure, the variational formulation of balancedness in linear systems theory and the DLN reduces to the minimization of a unitarily invariant squared norm on a group orbit. This problem has been solved by the Kempf-Ness theorem in GIT~\cite{Kempf}. Thus, balancedness theorems in both fields are consequences of the Kempf-Ness theorem. However, the Kempf-Ness theory does not provide quantitative rates of convergence or practical balancing algorithms. To this end, we use the Riemannian geometry of fibers to construct balancing flows with rigorous convergence guarantees (Theorem~\ref{thm:reg-flow} and Theorem~\ref{thm:ness-flow}).


This unification also allows us to reflect on common themes in the conceptual foundations of deep learning and linear systems theory. As Kalman writes in~\cite{Kalman-can}, a dynamical system may be described in two distinct ways: (i) by means of state variables (a model such as Newton's laws or equation~\eqref{eq:ls1}) and (ii) by input-output relations (a black box whose inner workings are opaque to the user but produces data such as $H(z)$). For the vast number of users of deep learning, it is the input-output relation that matters. However, for designers of the architecture of neural networks and for a complete scientific understanding of deep learning, it is necessary to understand training dynamics from first principles.



Our main finding then is that in both deep learning and linear systems theory, balanced (manifolds and realizations) correspond to optimal descriptions of the input-output relations in terms of the parameters of the model.



\subsection{Regularizing flows and the learning flow}
\label{subsec:reg-flow}
Recall that the state space for a gradient flow is a Riemannian manifold. Thus, in order to define gradient flows on $\fiberX$ and $\balance$, we must equip them with a Riemannian metric. 

It is only the complex case that requires some comment. When $\mathbb{F}=\C$, we equip $\Md(\C)$ with the inner-product
\begin{equation}
    \label{eq:ipc-1}
    \langle A,B  \rangle =  \mathrm{Re}\,\Tr \left( A^* B\right).
\end{equation}
Writing $A=a+ib$ and $B=c+id$ for the real and imaginary parts, we see that 
\begin{equation}
    \label{eq:ipc-2}
    \langle A, B \rangle =  \Tr \left( a^Tc+ b^Td\right),
\end{equation}
so that $\Md(\C)$ has the same norm as the real vector space $\Md(\R)\times \Md(\R)$. Similarly, we define the norm on $\Md^N(\C)$
\begin{equation}
    \label{eq:ipc}
    \langle \mathbf{A}, \mathbf{B} \rangle = \sum_{k=1}^N \mathrm{Re}\,\Tr \left( A_k^* B\right).
\end{equation}
This inner-product defines $\Md^N(\C)$ as a Euclidean space. Since the fiber $\fiberX$ and the balanced variety $\balance$ are defined as the solutions to the algebraic equations~\eqref{eq:balanced-intro1} and~\eqref{eq:balanced-intro1}, they are embedded Riemannian manifolds of the Euclidean space $\Md^N(\mathbb{F})$ under natural assumptions on $X$ (for example, when $X$ has full rank). We denote the resulting Riemannian manifolds $(\fiberX,\iota)$ and $(\balance,\iota)$ respectively. 


We study two complementary gradient flows on these manifolds:
\begin{enumerate}  
\item {\em Regularizing flows}: This is the gradient flow of a regularizer $R(\ww)$ on the Riemannian manifold $(\fiberX,\iota)$. Theorem~\ref{thm:intro} corresponds to the regularizer $R(\ww)= \tfrac{1}{2}\|\ww\|_2^2$ and the gradient flow
\begin{equation}
\label{eq:reg-gradient1}
    \dot{\mathbf{W}}= -\frac{1}{2}\mathrm{grad}\,\|\mathbf{W}\|_2^2, \quad \mathbf{W}\in (\fiberX,\iota).
\end{equation}
We also consider the gradient flow of the squared moment map, $R(\ww)= \tfrac{1}{2}\|\bfG\|^2$ (see Remark~\ref{rem:ness-flow} below). Theorem~\ref{thm:intro} extends to include the Schatten $p$-norms as regularizers for $1<p<\infty$. It does not yet include the important case $p=1$ (see Theorem~\ref{thm:intro-revised} in Section~\ref{sec:kempf-ness}). 
\item {\em The learning flow}: This is the gradient flow of the loss function $E(X)$ on $(\balance,\iota)$
\begin{equation}
\label{eq:learn-gradient1}
    \dot{\mathbf{W}}= -\mathrm{grad} \,E(X(\ww)), \quad \ww\in (\balance,\iota).
\end{equation}
\end{enumerate}
We show that the regularizing flow for the $L^2$-regularizer is exactly solvable in the following sense. 
\begin{theorem}
\label{thm:reg-flow} Assume $X$ has full rank and $\ww(t)$ solves equation~\eqref{eq:reg-gradient1} with initial condition $\ww_0 \in \fiberX$. Then the moments $\mathbf{G}(t)$ satisfy
\begin{equation}
    \label{eq:reg-gradient2}
    \mathbf{G}(t) = \mathbf{G}_0 e^{-2t}, \quad t \geq 0.
\end{equation}
\end{theorem}
Thus, the moment map $\ww \mapsto \mathbf{G}$ reduces the regularizing flow to scaling at a uniform rate. Since the balanced variety is the inverse image $\mathbf{G}^{-1}\{\mathbf{0}\}$, Theorem~\ref{thm:reg-flow} establishes attraction to the balanced manifold at a constant rate. 
\begin{remark}
\label{rem:ind-proof}
The methods that underly Theorem~\ref{thm:intro} and Theorem~\ref{thm:reg-flow} are different. Theorem~\ref{thm:intro} is an application of the Kempf-Ness theorem and thus relies on ideas primarily from algebraic geometry. On the other hand, Theorem~\ref{thm:reg-flow} relies on explicit matrix computations that reflect the underlying Riemannian geometry. 
\end{remark}
\begin{remark}
\label{rem:low-rank-scaling}
The assumption on rank may be relaxed in both Theorem~\ref{thm:intro} and Theorem~\ref{thm:reg-flow}. Theorem~\ref{thm:intro} requires that we work with group orbits so that we may apply the Kempf-Ness theorem. On the other hand, Theorem~\ref{thm:reg-flow} requires only that we work on a Riemannian manifold (in particular, the calculations in Section~\ref{sec:reg-flow} may be generalized to the setting of rank $r<d$). When $X$ has full rank, both these conditions are true. We focus on this situation so that our calculations are most transparent.

The underlying algebraic geometry, including the foliation by rank, has been analyzed in generality by Shewchuk and Bhattacharya~\cite{Shewchuk} and Pepin Lehalleur and Rim\'{a}nyi~\cite{Lehalleur}. While these results parallel one another, the methods are different: the article~\cite{Shewchuk} emphasizes linear algebra, whereas~\cite{Lehalleur} organizes these calculations using the language of quiver representations. The relationship between these geometric structures and the loss landscape for the quadratic loss function is a topic of current interest.
\end{remark}

\begin{remark}
\label{rem:learn-flow}
The learning flow~\eqref{eq:learn-gradient1} on $(\balance,\iota)$ is equivalent to the gradient flow
\begin{equation}
    \label{eq:learning-flow} \dot{X} = - \mathrm{grad}_{g^N} E(X), \quad X \in (\Md,g^N).
\end{equation}
Here the Riemannian manifold $(\Md,g^N)$ is obtained by Riemannian submersion from $(\balance,\iota)$ through the map $\ww \mapsto X$. The metric may be described explicitly~\cite{GM-dln}.

This statement is a synthesis of results from~\cite{ACH,Bah,MY-dln,MY-RLE} that identifies the learning flow as an equilibrium thermodynamic  process.
\end{remark}

\begin{remark}
\label{rem:reg-flow}
Our regularizing flow~\eqref{eq:reg-gradient1} differs from what has been used in linear systems theory. Gradient flows on $\GLn$ that balance a triple $(A,B,C)$ (`balancing flows') have been studied by Helmke and Moore~\cite{Helmke-Moore}. These gradient flows were inspired by Brockett's double-bracket flow on $O_n$~
\cite{Brockett}. However, these works use the normal metric on $\GLn$ and $O_n$ respectively, not the induced metric $\iota$. In our view, it is necessary to use the induced metric instead because (i) this conforms to the Euclidean metric used in practice for deep learning; (ii) it allows us to include noise in a geometrically natural manner using Riemannian Langevin equations (cf. \S \ref{subsec:noise}). 
\end{remark}
\begin{remark}
\label{rem:ness-flow}
Kirwan and Ness studied the Morse theory of the squared moment map in their pioneering work on GIT~\cite{Kirwan,Ness}. Ness studied the related gradient flow~\cite[\S 3-7]{Ness} and analyzed its convergence properties. In our context, this regularizing flow, which we term the {\em Kirwan-Ness\/} flow, is given by~\footnote{The reader should note that we use a different normalization from Ness. We find it more convenient to use $\tfrac{1}{2}\|\mathbf{G}\|_2^2$ rather than $\|\mathbf{G}\|_2^2$.}
\begin{equation}
\label{eq:reg-gradient5}
    \dot{\mathbf{W}}= -\frac{1}{2}\mathrm{grad}\,\|\mathbf{G}\|_2^2, \quad \mathbf{W}\in (\fiberX,\iota).
\end{equation}
A striking feature of the DLN is that this flow is given explicitly by the cubic system
\begin{equation}
    \label{eq:ness-flow-coords}
    \dot{W}_k = 2(W_kG_{k-1}-G_{k}W_k), \quad 1\leq k \leq N,
\end{equation}
with the convention $G_N=G_0=0$. This allows us to adapt the general theory from GIT for the purpose of balancing in deep learning (see Section~\ref{sec:ness} below). 

The Kirwan-Ness flow also provides a technical relationship between deep learning and mathematical physics through the common structure of Yang-Mills theory as discussed in Uhlenbeck's recent work~\cite{Atiyah-Bott-YM,Uhlenbeck2026}. We note that mathematical physics techniques such as the renormalization group and diagrammatic expansions have been used to study deep learning, but rigorous mathematical justification of these ideas is still limited~\cite{Hanin-book}. 
\end{remark}

\begin{figure}
    \centering
    \includegraphics[trim={425 500 1200 225}, clip, width=.7\linewidth]{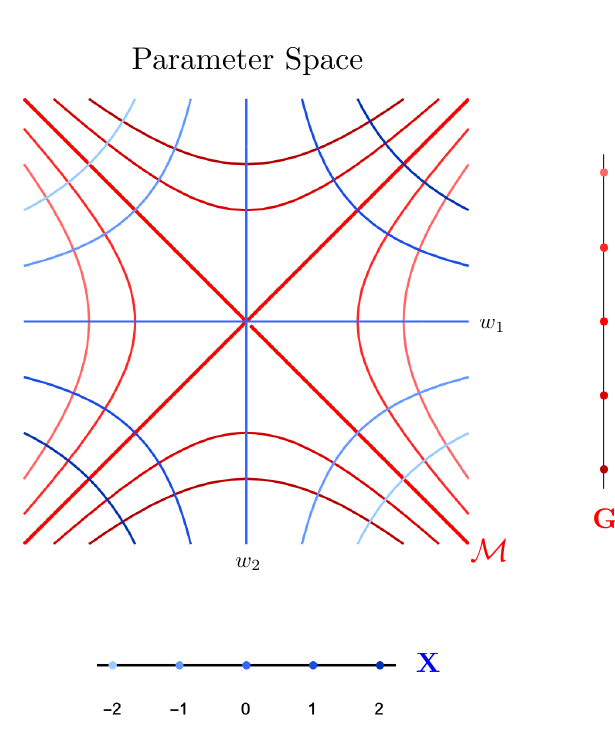}
    \caption{This figure describes the orthogonal foliation of $\Md^N$ by the balanced varieties $\balance_{\mathbf{G}}$ and fibers $\fiberX$ in the simplest case ($d=1$ and $N=2$ and real matrices). The regularizing flow lies on the hyperbola $w_2w_1=x$. The learning flow lives on the asymptotes $w_2 = \pm w_1$. It is intuitively clear that the minimizers of $|w|^2$ on the fiber $w_2w_1=x$ are the points $(\pm \sqrt{x},\pm \sqrt{x})$. Theorem~\ref{thm:intro} establishes the analogous property in general.}
    \label{fig:dln}
\end{figure}

\begin{figure}[h]
    \centering
    \includegraphics[trim={175 75 1100 75}, clip, width=.4\linewidth]{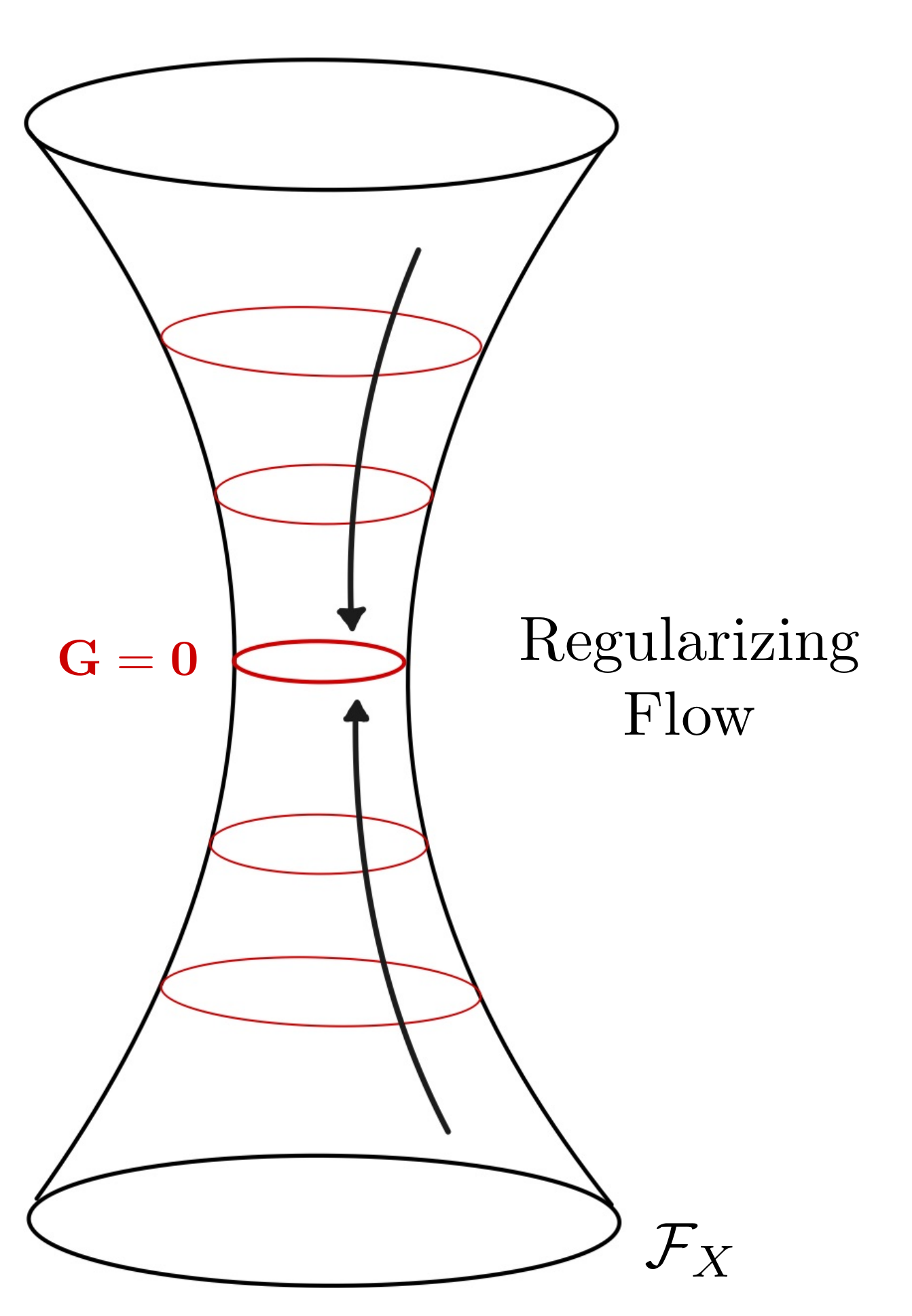}
    \caption{The regularizing flow (see Theorem~\ref{thm:reg-flow}) on $\fiberX$. When $d\geq 2$ the fiber $\fiberX$ is sliced by the moments $\bfG$ into topologically equivalent components $\fiberX \cap \balance_\bfG$. The regularizing flow evolves the slices at a uniform exponential rate towards the minimizing orbit $\orbitx$ corresponding to $\bfG=\mathbf{0}$.}
    \label{fig:reg-flow}
\end{figure}

\begin{figure}[h]
    \centering
    \includegraphics[trim={500 275 600 200}, clip, width=.6\linewidth]{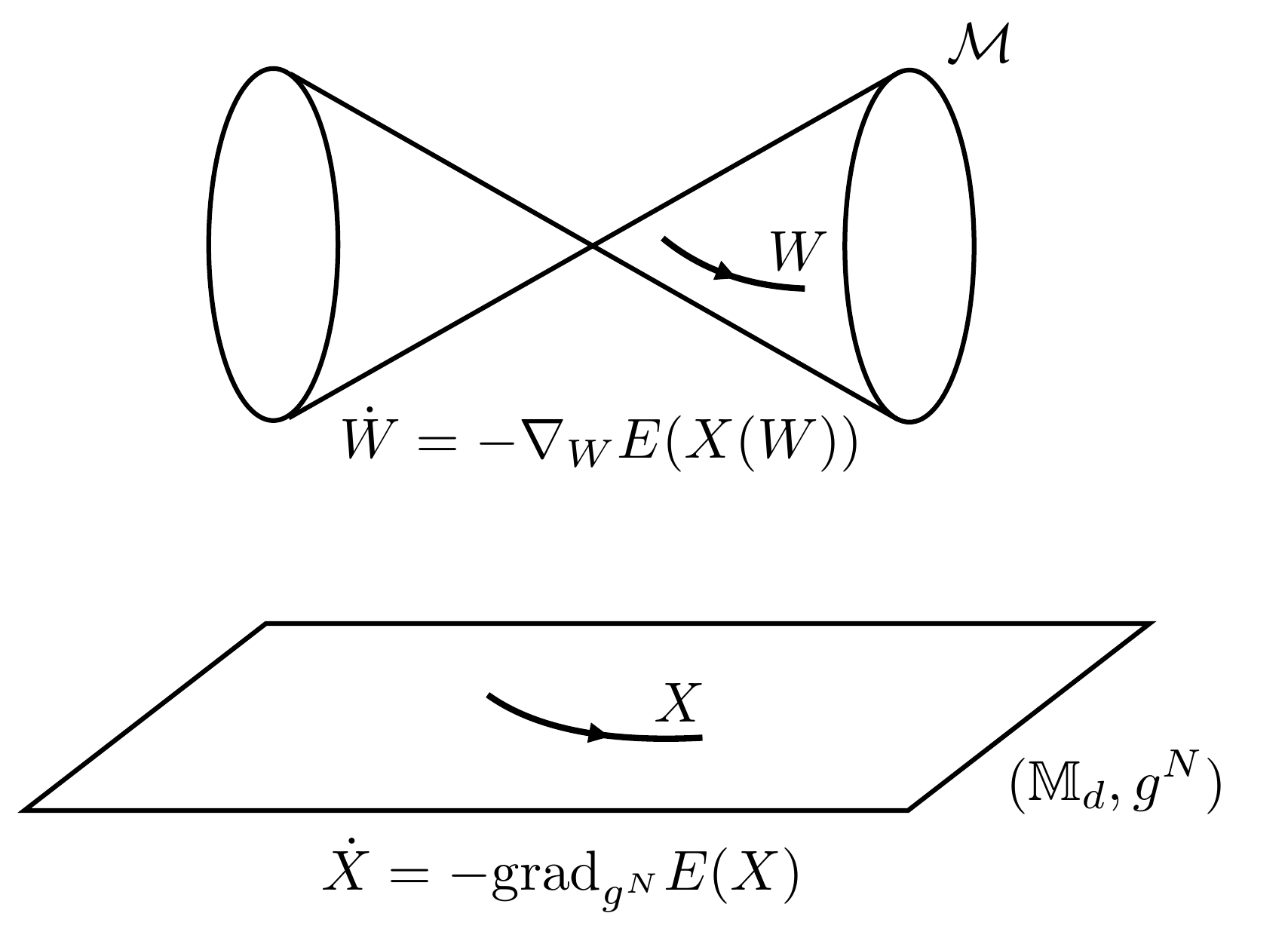}
    \caption{The learning flow. There are two equivalent descriptions: the balanced manifold $\balance$  is invariant under the gradient flow $\dot\ww =-\nabla_\ww E(X(\ww)$ of the cost function. Further, the dynamics of the end-to-end matrix $X$ are given by the Riemannian gradient flow $\dot{X}=-\grad_{g^N} E(X)$ on $(\Md,g^N)$ where the manifold $(\Md,g^N)$ is obtained by Riemannian submersion from $(\balance,\iota)$.}
    \label{fig:learn-flow}
\end{figure}

\subsection{Organization of the paper}
We first study the Riemannian geometry of fibers, introduce the regularizing flows and establish their convergence (Theorem~\ref{thm:reg-flow} and Theorem~\ref{thm:ness-flow}) in Section~\ref{sec:reg-flow}. This is followed by a discussion of the Kempf-Ness theorem, its generalization, the Azad-Loeb theorem, and the application of GIT to the DLN. The paper concludes with a discussion of a new dynamic paradigm suggested  by the methods of this paper.

\section{Regularizing flows}
\label{sec:reg-flow}

In this section, we develop the Riemannian geometry of $\fiberX$, compute the gradient $\grad \|\ww\|^2$ and prove Theorem~\ref{thm:reg-flow}. We emphasize concrete matrix computations. In the next section, we place these computations within the abstract conception of the Kempf-Ness theorem to establish Theorem~\ref{thm:intro}. All calculations are performed under the assumption that $\mathbb{F}=\C$. The case $\mathbb{F}=\R$ follows naturally. 




\subsection{$\fiberX$ is a $\GLdC^{N-1}$ orbit}
The assumption that $X$ has full rank allows us to characterize $\fiberX$ as a group orbit.
\begin{lemma}
\label{le:group-orbit}
Assume $X$ has full rank. The point $\ww \in \fiberX$ if and only if it is of the form $\bfA \cdot \bfC$ for some $\bfA\in \GLdC^{N-1}$.
\end{lemma}
\begin{proof}
Let $X=Q_N \Sigma Q_0^*$ denote the SVD of $X$ and let $\Lambda = \Sigma^{1/N}$. Then 
\begin{equation}
\nonumber
\bfC=(Q_N \Lambda, \Lambda, \ldots, \Lambda Q_0^*), \quad\mathrm{and}\quad \bfA\cdot \bfC = (Q_N \Lambda A_{N-1}^{-1}, A_{N-1} \Lambda A_{N-2}^{-1},\ldots, A_1 Q_0^*).   
\end{equation}
Given $\ww =(W_N,\ldots,W_1) \in \fiberX$, we know that each $W_p$ has full rank since $W_N \cdots W_1=X$. Thus, we may determine $\bfA$ in sequence. First, we choose $A_{N-1}$ such that $Q_N\Lambda A_{N-1}^{-1}= W_N$ by setting $A_{N-1}= Q_N\Lambda W_N^{-1}$. Next, we choose $A_{N-2}$ so that $A_{N-1} \Lambda A_{N-2}^{-1}=W_{N-1}$ and so on. 

Conversely, given $\bfA\in \GLdC^{N-1}$, it is clear that $\bfA\cdot \bfC \in \fiberX$.
\end{proof}

\subsection{Differential geometry of $\fiberX$}
The tangent space to $\fiberX$ is computed as follows. Given $\bfa \in \gldC^{N-1}$ we define a curve through the identity using
\begin{equation}
    \label{eq:tangent1}
    \bfA(\tau) = (e^{\tau a_{N-1}}, \cdots, e^{\tau a_1}) := e^{\tau \bfa}, \quad \tau \in (-\infty,\infty).
\end{equation}
Then the tangent space $T_{\bfW}\fiberX$ consists of the vectors 
\begin{equation}
    \label{eq:tangent2}
    \bfw_\bfa := \left. \frac{d}{d\tau} e^{\tau \bfa} \cdot \bfW\right|_{\tau=0}, \quad \bfa \in \gldC^{N-1}.
\end{equation}
We substitute in equation~\eqref{eq:group-action1} to find  
\begin{equation}
    \label{eq:tangent3}
    \bfw_\bfa = (-W_N a_{N-1}, a_{N-1}W_{N-1} - W_{N-1}a_{N-2}, \cdots, a_1 W_1), \quad \bfa \in \gldC^{N-1}.
\end{equation}
Consider a smooth function $F:\fiberX \to \R$. We define the differential $dF$ by its action on $T_\ww \fiberX$ as follows:
\begin{equation}
    \label{eq:differential1}
    dF(\ww) \bfw_\bfa = \left. \frac{d}{d\tau} F(\ww(\tau))\right|_{\tau=0}, \quad \ww(\tau) = e^{\tau \bfa} \cdot \bfW.
\end{equation}
\begin{lemma}
    \label{le:moments} Let $F(\ww)= \tfrac{1}{2}\|\ww\|_2^2$. Then
    \begin{equation}
    \label{eq:differential2} 
     dF(\ww) \bfw_\bfa = \frac{1}{2}\sum_{k=1}^{N-1} \Tr\left(G_k (a_k + a_k^*)\right) = \sum_{k=1}^{N-1} \mathrm{Re}\, \Tr\left(G_k^* a_k\right).
    \end{equation}
\end{lemma}
\begin{proof}
Consider a curve $\ww(\tau)$ with $\ww(0)=\ww$ and $\dot{\ww}(0)=\bfw_\bfa$. We differentiate the expression
\[ 2F(\ww(\tau))= \sum_{k=1}^N \Tr\left(W_k^*(\tau)W_k(\tau)\right) \]
with respect to $\tau$ and evaluate it at $\tau=0$ to obtain
\begin{eqnarray}
&& \lefteqn{2 dF(\ww) \bfw_\bfa= -\Tr\left( a_{N-1}^*W_N^*W_N + W_N^*W_N a_{N-1}\right)} \\
\nonumber
&& + \sum_{k=2}^{N-1} \Tr \left( (W_k^*a_k^* -a^*_{k-1}W_k^*)W_k + W_k^*(a_kW_k -W_ka_{k-1}) \right) + \Tr \left(W_1^* a_1^*W_1+ W_1^*a_1 W_1\right) \\
\nonumber
&& = \sum_{k=1}^{N-1} \Tr \left( (W_kW_k^*-W_{k+1}^*W_{k+1}) (a_k +a^*_k)\right)  = \sum_{k=1}^{N-1} \Tr \left( G_k (a_k +a^*_k)\right).
\end{eqnarray} 
The second equality in equation~\eqref{eq:differential2} follows because $G_k=G_k^*$.
\end{proof}
\subsection{Riemannian geometry of $\fiberX$}
\label{subsec:riemann}
The inner product $\langle \bfw_\bfa , \bfw_\bfb \rangle$ between two vectors $\bfw_\bfa$ and $\bfw_\bfb$ in $T_{\bfW}\fiberX$ is induced by the inner product on $\Md^N$. We have
\begin{eqnarray}
    \nonumber
    \lefteqn{\langle \bfw_\bfb, \bfw_\bfa \rangle = \mathrm{Re} \, \Tr\left(b_{N-1}^* W_N^*W_N a_{N-1}\right) + }
    \\  \label{eq:tangent4}
    && \sum_{k=2}^{N-1}\mathrm{Re} \, \Tr \left( (W_k^*b_k^* - b_{k-1}^*W_k^*)    (a_{k}W_{k} - W_{k}a_{k-1}) \right) 
    \\ \nonumber
    && + \mathrm{Re} \, \Tr\left(W_1 W_1^* b_1^* a_1 \right).
\end{eqnarray}
The Riemannian manifold $(\fiberX,\iota)$ is completely prescribed by our characterization of $\fiberX$ as a smooth manifold along with the inner-product $\langle \cdot, \cdot \rangle$. 

We express the inner product using the following linear operation.
\begin{defn}
\label{def:H}
Given $\ww \in \fiberX$, define the linear transformation $\mathbf{H}: \gldC^{N-1} \to \gldC^{N-1}$ where $\mathbf{H} =(H_{N-1},\cdots, H_1) $ and $H_k=H_k(\bfc) \in \Md$ is defined by  
 \begin{equation}
     \label{eq:defH2}
     H_k(\bfc)= - W_k c_{k-1} W_k^* + c_k W_k W_k^*  + W^*_{k+1}W_{k+1} c_k - W^*_{k+1}c_{k+1}W_{k+1}.
 \end{equation}
 We adopt the convention that $c_0=c_N=0$.
\end{defn}
\begin{lemma}
    \label{le:defH}
The inner product $\langle \bfw_\bfb, \bfw_\bfa \rangle$ may be rewritten as
\begin{equation}
    \label{eq:defH1}
    \langle \bfw_\bfb, \bfw_\bfa \rangle = \sum_{k=1}^{N-1}\mathrm{Re} \, \Tr (H_k(\bfb)^* a_k) = \sum_{k=1}^{N-1}\mathrm{Re} \, \Tr (b_k^* H_k(\bfa)). 
\end{equation}
\end{lemma}
\begin{proof}
This Lemma is just a convenient reorganization of the terms in equation~\eqref{eq:tangent4}. 

Let us prove the first equality. Collect the terms involving $a_k$ in the first equality in equation~\eqref{eq:tangent4} to obtain 
\begin{equation}
    \label{eq:defH3}
\Tr \left( (W_k^*b_k^* - b_{k-1}^*W_k^*)a_{k}W_{k}\right) -   \Tr \left( (W_{k+1}^*b_{k+1}^* - b_{k}^* W_{k+1}^*)W_{k+1}a_{k} \right).
\end{equation}
Since the trace is cyclic, we may rewrite the above expression as
\begin{equation}
    \nonumber
\Tr \left( (W_kW_k^*b_k^* - W_kb_{k-1}^*W_k^*)a_{k}\right) -   \Tr \left( (W_{k+1}^*b_{k+1}^*W_{k+1} - b_{k}^* W_{k+1}^*W_{k+1})a_{k} \right). 
\end{equation}
This is $\Tr(H_k(\bfb)^*a_k)$. We sum over $k$ to obtain the first equality in equation~\eqref{eq:defH1}. The proof of the second equality is similar.  
\end{proof}
The meaning of the matrices $H_k$  may be clarified as follows. 
\begin{lemma}
    \label{le:grad-moment-2} 
    $\mathbf{H}(\bfc)+ \mathbf{H}^*(\bfc)$ is the pushforward of $\bfw_\bfc \in T_\ww \fiberX$ under the map $\ww \mapsto \bfG$. 
  \end{lemma}
Here we use the following notation for the pushforward
  \begin{equation}
  d\bfG \,\bfw_\bfc = (dG_{N-1}\bfw_\bfc,\ldots,dG_1\bfw_\bfc).    
  \end{equation}
\begin{proof}
Fix $\ww \in \fiberX$ and consider a curve $\ww(\tau)$ such that $\ww(0)=\ww$ and $\dot{\ww}(0)=\bfw_\bfc$. Then by definition $dG_k \bfw_\bfc= \dot{G}_k$ where the curve $\bfG(\tau)$ is defined through the moment map applied to $\ww(\tau)$. Thus, to prove the lemma it is enough to show that
   \begin{equation}
    \label{eq:def-grad5}
        \dot{G}_k = H_k +H_k^*, \quad 1\leq k \leq N-1.        
    \end{equation}
We differentiate $G_k = W_k W_k^* - W_{k+1}^* W_{k+1}$ with respect to $\tau$ to find
\begin{equation}
\label{eq:grad-moment3}
\dot{G}_k = \dot{W}_k W_k^* + W_k \dot{W}_k^* - \dot{W}_{k+1}^* W_{k+1} - W_{k+1}^* \dot{W}_{k+1}.
\end{equation}
Set $\tau=0$ and substitute $\dot{W_k} = c_k W_k - W_k c_{k-1}$ into equation~\eqref{eq:grad-moment3} to obtain
\begin{eqnarray}
\label{eq:grad-moment4}
\lefteqn{\dot{G}_k = (c_k W_k- W_k c_{k-1}) W_k^* + W_k (c_k W_k- W_k c_{k-1})^*}
\\ \nonumber
&& - (c_{k+1} W_{k+1}- W_{k+1} c_{k})^* W_{k+1} - W_{k+1}^*(c_{k+1} W_{k+1}- W_{k+1} c_{k})
\\ \nonumber
&& = c_k W_k W_k^*- W_k c_{k-1} W_k^* + W_k W_k^*c_k^*- W_k c_{k-1}^*W_k^* 
\\ \nonumber 
&& -W_{k+1}c_{k+1}^*W_{k+1}^*+ c_{k}^*W_{k+1}^* W_{k+1} - W_{k+1}^*c_{k+1} W_{k+1} + W_{k+1}^*W_{k+1} c_{k}.
\end{eqnarray}
We now rearrange terms to obtain  the identity~\eqref{eq:def-grad5}.
\end{proof}

\subsection{The regularizing flow}
The gradient of $F:\fiberX\to \R$, denoted $\grad F$, is the unique tangent vector in $T_\ww \fiberX$ such that  
\begin{equation}
    \label{eq:def-grad}
    \langle \grad F , \bfw_\bfa \rangle = dF \, \bfw_\bfa, \quad \bfw_\bfa \in T_\ww \fiberX.
\end{equation}
In coordinates, $\grad F$ is obtained by solving a linear system. By the characterization of $T_\ww \fiberX$ and the non-degeneracy of the inner-product $\langle \cdot, \cdot \rangle$, we see that
\begin{equation}
    \label{eq:def-grad2} \grad F = \bfw_\bfb
\end{equation}
for a unique $\bfb \in \gldC^{N-1}$ that is determined by the linear system
\begin{equation}
    \label{eq:def-grad3} \langle \bfw_\bfb, \bfw_\bfa \rangle = dF \, \bfw_\bfa, \quad \bfw_\bfa \in T_\ww \fiberX.
\end{equation}
The solution to this system completes the prescription of the gradient flow of $F$
\begin{equation}
    \label{eq:def-grad11} \dot\ww = - \grad F, \quad \ww \in \fiberX.
\end{equation}

Let us now specialize to the case where $F(\ww) = \tfrac{1}{2} \|\ww\|_2^2$ is the $L^2$ regularizer.
We now use Lemma~\ref{le:moments} and Lemma~\ref{le:defH} to obtain
\begin{lemma}
    \label{le:grad} $\grad \tfrac{1}{2} \|\ww \|^2 = \bfw_\bfb$ where $\bfb$ solves the block tridiagonal system 
\begin{equation}
\label{eq:def-grad12}
H_k(\bfb)= G_k, \quad 1\leq k \leq N-1.    
\end{equation}
\end{lemma}
The block tridiagonal structure is as follows.  Since $G_k=G_k^*$ we have
\begin{equation}
\label{eq:def-grad12a}
H_k(\bfb)^* = H_k^*(\bfb), \quad 1\leq k \leq N-1.   
\end{equation}
We now use Definition~\ref{def:H} to see that equation~\eqref{eq:def-grad12} is equivalent to 
    \begin{eqnarray}
        \nonumber
        && -W_k(b_{k-1}+b_{k-1}^*) + \\
        \label{eq:def-grad4}
        && b_k W_k W_k^* + W_k W_k^* b_k^* + W_{k+1}^*W_{k+1}b_k + b_k^* W_{k+1}^* W_{k+1} \\
        \nonumber
        && - W_{k+1}^*(b_{k+1}+ b_{k+1}^*) W_{k+1}= 2 G_k, \quad 1 \leq k \leq N-1.
    \end{eqnarray}
\begin{proof}
We use equation~\eqref{eq:differential2} and equation~\eqref{eq:defH1} to obtain the identity
\begin{equation}
    \mathrm{Re} \Tr(H_k^*(\mathbf{b}) a_k) = \mathrm{Re} \Tr \left(G_k^* a_k)\right) =0, \quad 1\leq k \leq N-1.
\end{equation}  
This identity holds for all $a_k \in \Md$. The theorem follows immediately because of the degeneracy of the inner-product defined in~\eqref{eq:ipc-1}.
\end{proof}
The form of these equations allows us to linearize the regularizing flow.
\begin{proof}[Proof of Theorem~\ref{thm:reg-flow}]
By Lemma~\ref{le:grad}, $\grad \tfrac{1}{2}\|\ww\|^2 = \bfw_\bfb$ where $\bfb$ satisfies equation~\eqref{eq:def-grad12}. Therefore, by Lemma~\ref{le:grad-moment-2} and Lemma~\ref{le:grad} 
\begin{equation}
\label{eq:dotG}
\dot{\bfG} = d\bfG \bfw_\bfb = -2 \bfG.     
\end{equation} 
\end{proof}
\begin{remark}
\label{rem:proof-reg-flow}
Our proof of Theorem~\ref{thm:reg-flow} relies on explicit computations with the Riemannian manifold $(\fiberX,\iota)$. We present these calculations since they allow us to consider other gradient flows on $\fiberX$, such as the Kirwan-Ness flow introduced below. However, the reader should note that the cancellations that lead to the closed form for $\dot{\bfG}$ have a simple geometric origin. 

We first observe that the gradient of $\tfrac{1}{2}\|\ww\|_2^2$ in $\Md^N$ is simply $\ww$. The gradient may then be decomposed into two components $\bfw_\bfb=\grad \tfrac{1}{2}\|\ww\|^2 \in T_\ww\fiberX$ and $\ww^\perp:=\ww-\bfw_\bfb \in T_\ww\fiberX^\perp$.  The conservation laws for $\bfG$ are due to the fact that $T_\ww\fiberX^\perp$ lies in the nullspace of $d\bfG$. Thus,
\[ d\bfG \, \bfw_\bfb = d\bfG \,\ww = -2\bfG,\]
after an easy calculation. 
\end{remark}
\subsection{The Kirwan-Ness flow}
\label{sec:ness}
The Morse theory of the squared moment map is a fundamental concept in Kirwan and Ness' work~\cite{Kirwan,Ness}. We derive its gradient flow for the DLN by applying the calculations of Section~\ref{subsec:riemann} to the function
\begin{equation}
    \label{eq:ness-function}
    \frac{1}{2} \|\bfG\|_2^2 = \frac{1}{2} \sum_{k=1}^{N-1} \Tr(G_k^*G_k) = \frac{1}{2} \sum_{k=1}^{N-1} \Tr(G_k^2).
\end{equation}
Set $\bfc=\bfG$ in definition~\ref{eq:defH1} to obtain the matrices
\begin{equation}
     \label{eq:defH7}
     H_k(\bfG)= - W_k G_{k-1} W_k^* + G_k W_k W_k^*  + W^*_{k+1}W_{k+1} G_k - W^*_{k+1}G_{k+1}W_{k+1}.
 \end{equation}
 \begin{theorem}
     \label{thm:ness-flow}
The gradient flow of $\tfrac{1}{2}\|\bfG\|^2_2$ is 
\begin{equation}
\label{eq:ness-flow1}
    \dot{\ww} = -  \bfw_{2\bfG},
\end{equation}
or equivalently, the cubic system~\eqref{eq:ness-flow-coords}. The corresponding evolution of the moments is given by
\begin{equation}
\label{eq:ness-flow2}
    \dot{\bfG} = - 2\left( \mathbf{H}(\bfG) + \mathbf{H}(\bfG)^*\right).
\end{equation}
\end{theorem}
\begin{proof}
For convenience of notation, let $F(\ww)=\tfrac{1}{2}\|\bfG \|_2^2$. Then 
\begin{equation}
\label{eq:ness-flow3}
dF(\ww)\bfw_\bfa = \sum_{k=1}^{N-1} \Tr \left( G_k dG_k \bfw_\bfa\right)=   \sum_{k=1}^{N-1} \Tr \left( G_k (H_k(\bfa)+ H_k(\bfa)^*\right).    
\end{equation} 
On the other hand, if $\grad F = \bfw_\bfb$ then by Lemma~\ref{eq:defH1}
\begin{equation} 
\label{eq:ness-flow4}
\langle \grad F, \bfw_\bfa \rangle = \langle \bfw_\bfb, \bfw_\bfa \rangle= \sum_{k=1}^{N-1} \mathrm{Re} \Tr \left(b_k^*H_k(\bfa)\right). 
\end{equation}
Thus, we have the identity
\begin{equation} 
\label{eq:ness-flow5}
\sum_{k=1}^{N-1} \mathrm{Re}  \Tr \left(b_k^*H_k(\bfa)\right) =  \sum_{k=1}^{N-1} \Tr \left( G_k (H_k(\bfa)+ H_k(\bfa)^*\right) = 2 \sum_{k=1}^{N-1} \mathrm{Re} \Tr \left( G_k (H_k(\bfa)\right). 
\end{equation}
When $X$ has full rank, $\mathbf{H}$ is an isomorphism. Thus, $b_k=2G_k^*=2G_k$. Equation~\eqref{eq:ness-flow2} follows from Lemma~\ref{le:grad-moment-2} with $\bfc=2\bfG$.
\end{proof}

%
\begin{remark}
While we assumed that $\mathbf{W}\in \fiberX$ in order to define the gradient flow~\eqref{eq:ness-flow1} we see that the flow is given by the cubic system~\eqref{eq:ness-flow-coords} which is defined on all of $\Md^N$. Thus, the flow may be easily implemented numerically. Its convergence as $t \to \infty$ follows from the general theory for the gradient flow of the squared moment map. An excellent exposition is provided by Lerman and adapting~\cite[Thm. 1]{Lerman} we have
\begin{corollary}
\label{cor:ness-flow}
The $\omega$-limit set $\omega(\mathbf{W})$ under the Kirwan-Ness flow for any $\mathbf{W}\in \Md^N$  is a single point on the balanced variety $\mathcal{M}_0$. 
\end{corollary}
\end{remark}
The Kirwan-Ness flow presents an interesting contrast with the regularizing flow for $\tfrac{1}{2}\|\ww\|^2$. Lemma~\ref{le:grad} shows that when we consider the functional $\tfrac{1}{2}\|\ww\|_2^2$, the gradient $\grad \tfrac{1}{2} \|\ww\|_2^2 =\bfw_\bfb$ where $\bfb$ is given implicitly through the solution of the linear system ~\eqref{eq:def-grad12}. This makes numerical implementations of this regularizing flow subtle, since one must solve for $\bfb$ at each step. However, despite the implicit nature of the regularizing flow, Theorem~\ref{thm:reg-flow} tells that $\bfG$ evolves by pure scaling.

In contrast, the Kirwan-Ness flow~\eqref{eq:ness-flow1} is given explicitly by~\eqref{eq:ness-flow-coords} and is easy to implement numerically. On the other hand, while it is immediate from the definition of the gradient flow~\eqref{eq:ness-flow1} that 
\begin{equation}
\label{eq:ness-flow7}
    \frac{d}{dt}\tfrac{1}{2}{\|\bfG\|_2^2} =  - 2\|\bfw_\bfG \|_2^2 ,
\end{equation}
we do not have a closed evolution equation for $\bfG$.

\section{The Kempf-Ness theorem and the DLN}
\label{sec:kempf-ness}
\subsection{Overview}
We first review the abstract framework of the Kempf-Ness theorem. The proof of Theorem~\ref{thm:intro} reduces to a verification of the hypotheses of this theorem. We then discuss a more general class of minimization principles covered by 
the Azad-Loeb theorem. At present, our results do {\em not\/} include $L^1$-regularization (though see Theorem~\ref{thm:intro-revised} below).

\subsection{The Kempf-Ness theorem: abstract structure}
We summarize the abstract setup  following Helmke~\cite[\S 2]{Helmke}. The reader is also referred to~\cite{Ness} for finer results based on the gradient flow of the squared moment map. 

We assume given a complex reductive Lie group $\Ggroup$ with maximal compact subgroup $\Kgroup$ and a finite-dimensional complex vector space $V$. Examples are $\Ggroup =\GLdC$, $\Kgroup = U_d$ and $V = \C^d$. Let 
\begin{equation}
\label{eq:kn1}
\alpha: \Ggroup \times  V \to V
\end{equation}
denote a linear algebraic action of $\Ggroup$ on $V$. The orbit of a point $x \in V$ under the $\Ggroup$ action is the subset of $V$ given by
\begin{equation}
\label{eq:kn2}
\groupx  = \{g\cdot x \left| g \in \Ggroup\right. \}.    
\end{equation}
The stabilizer subgroup $\stabx$ is the subgroup of $\Ggroup$ that fixes $x$. That is, 
\begin{equation}
\label{eq:kn3}
\stabx  = \{g \in \Ggroup  \left| g\cdot x= x\right. \}.    
\end{equation}
On general grounds, the orbit $\groupx$ is a complex manifold that is biholomorphically equivalent to the symmetric space $\Ggroup/\stabx$. 

The Kempf-Ness theory studies the critical points of $\Kgroup$-invariant functions on $\groupx$. A function $\varphi: \groupx \to \C$ is $\Kgroup$-invariant if 
\[ \varphi(k\cdot y) = \varphi (y), \quad y \in \groupx, k \in \Kgroup.\]
A typical example of a $\Kgroup$-invariant function is a $\Kgroup$-invariant norm $\|\cdot \|$ on $V$. In particular, we may consider norms defined by a Hermitian inner-product $\langle \cdot, \cdot \rangle$. The norm is $\Kgroup$-invariant when 
\[ \langle k\cdot u, k \cdot v \rangle = \langle u, v \rangle, \quad k \in \Kgroup, \quad u,v \in V.\]

For any such norm, we consider the distance functions $\groupx \to \R$, $y \mapsto \|y\|^2$. Since $\groupx$ is a group orbit, we may also view this as a function 
\begin{equation}
    \label{eq:def-psi}
    \psi_x:\Ggroup \to \R, \quad g \mapsto \|g\cdot x\|^2. 
\end{equation}
The function $\psi_x$ is a $\Kgroup$-invariant function on $\Ggroup$. Let $e \in \Ggroup$ denote the identity. The derivative of $\psi_x$ at $e$ is computed as follows. Consider an element $a \in \Galgebra$ and the one-parameter subgroup $e^{\tau a} \in \Ggroup$, $\tau \in \R$. Then 
\begin{equation}
    \label{eq:def-psi2}
    d\psi_x(e)(a) = \left. \frac{d}{d\tau} \psi_x(e^{\tau a})\right|_{\tau=0}.
\end{equation}
Thus, $d\psi_x \in \Galgebra^*$ and vanishes when $a \in \Kalgebra$, the Lie algebra of $\Kgroup$.

\begin{defn}
The moment map $\moment$ is the function
\begin{equation}
    \label{eq:def-psi3}
    \moment: V \to \Galgebra^*/\Kalgebra^*, \quad x \mapsto d\psi_x(e).
\end{equation}
\end{defn}


We now state the Kempf-Ness theorem(s), making modest stylistic changes from the versions stated in~\cite{Helmke,Kempf}.
\begin{theorem}[Kempf-Ness]
\label{thm:kn1}
Assume given a linear algebraic action  $\alpha: \Ggroup \times V$ of a complex reductive group $\Ggroup$ on a finite-dimensional vector space $V$ and  a $\Kgroup$-invariant Hermitian norm on $V$. The following are equivalent:
\begin{enumerate}
    \item $\psi_x$ has a critical point on $\Ggroup$.
    \item $\psi_x$ has a minimum on $\Ggroup$.
    \item The orbit $\groupx$ is closed.
\end{enumerate}
\end{theorem}
\begin{theorem}[Kempf-Ness]
\label{thm:kn2}
Assume the hypotheses of Theorem~\ref{thm:kn1} and assume that $\groupx$ is closed. Then
\begin{enumerate}
    \item Every critical point of $\psi_x$ is a global minimum and the set of global minima is a unique $\Kgroup$-orbit.
    \item The Hessian of $\psi_x$ is positive semi-definite at each critical point on the $\Kgroup$-orbit, degenerating only in the directions tangent to the $\Kgroup$-orbit.
\end{enumerate}
\end{theorem}
\begin{remark}
 The Kempf-Ness theorem has been extended to real groups and vector spaces by Slodowy~\cite{Slodowy}. We do not state this theorem separately but we use it below.
\end{remark}
\begin{remark}
The reader may gain some intuitive insight into these theorems by considering Figure~\ref{fig:dln} and Figure~\ref{fig:reg-flow}. The group $\Ggroup$ here is the group of positive real numbers with the group action being $(w_2,w_1) \mapsto (w_2 A^{-1},Aw_1)$, $A \in \R_+$. The orbits that are not closed in Figure~\ref{fig:dln} are the semi-axes within the singular variety $w_2w_1=0$ (that is, either $w_1=0$ or $w_2=0$, but not both). 
\end{remark}
\begin{remark}
Ness used the gradient flow of $\|\moment\|^2$ to classify the non-closed orbits, further stratifying them according to the minimal and non-minimal critical points of $\|\moment\|^2$~\cite[Thm 6.2]{Ness}). This analysis motivated our introduction of the Kirwan-Ness flow in Section~\ref{sec:ness}.
\end{remark}

\subsection{Application of the Kempf-Ness theorem}
\begin{proof}[Proof of Theorem~\ref{thm:intro}]
We first note the equivalence between the assumptions of the Kempf-Ness theorem and group actions in the DLN. The vector space $V$ is $\Md^N(\C)$, the group $\Ggroup$ is $\GLdC^{N-1}$, the subgroup $\Kgroup$ is $U_d^{N-1}$ and the group action $\alpha: \Ggroup \times V \to V$ is the group action $\ww \mapsto \bfA \cdot \ww$ stated in equation~\eqref{eq:group-action1}. The norm $\|\ww\|_2^2$ is clearly $U_d^{N-1}$ invariant.
Thus, the groups, group action and norm satisfy the hypotheses of the Kempf-Ness theorem.

A somewhat more subtle hypothesis to verify is whether the fibers $\fiberX$ defined by the polynomial equation $W_N\cdot W_1 =X$ are indeed group orbits. When $X$ has full rank, Lemma~\ref{le:group-orbit} shows that $\fiberX$ is of the form $\Ggroup_x$ in the setup of the Kempf-Ness theorem. Thus, Theorem~\ref{thm:intro} follows for complex matrices. 

Similarly, we may also consider the vector space $\Md^N(\R)$, the group $\GLdR^{N-1}$, the subgroup $O_d^{N-1}$ and the group action $\ww \mapsto \bfA \cdot \ww$ as in equation~\eqref{eq:group-action1}. The norm $\|\ww\|_2^2$ is now $O_d^{N-1}$ invariant.
Thus, for the real DLN the groups, group action and norm satisfy the hypotheses of Slodowy's extension of the Kempf-Ness theorem. Again, the fiber $\fiberX$ is a group orbit when $X$ has full-rank. Thus, Theorem~\ref{thm:intro} holds for the real DLN.
\end{proof}
\begin{remark}
The moment map for the DLN follows from equations~\eqref{eq:def-psi}--~\eqref{eq:def-psi3} and Lemma~\ref{le:moments}. We find that 
\begin{equation}
    \label{eq:moment-DLN}
    \mu(\ww) = \bfG(\ww).
\end{equation}

The importance of working over $\Md^N(\C)$ first is that a moment map must be defined on a symplectic manifold. While both Theorem~\ref{thm:intro} and Theorem~\ref{thm:reg-flow} hold for $\Md^N(\R)$, the fiber $\fiberX$ is not in general a symplectic manifold for real matrices (it may not even be even-dimensional).
\end{remark}

\subsection{Hidden convexity}
\label{subsec:convexity}
The Kempf-Ness theorem may be seen as an assertion that the squared norm function $\psi_x: \Ggroup\to \R$ has properties analogous to a convex function. In fact, the proof of the theorem begins with a consideration of `special functions' on the line of the form  $\sum_{i=1}a_ e^{l_i x}$ where $a_i$ are positive numbers and the $l_i$ are arbitrary real numbers~\cite[\S 1]{Kempf}. Azad and Loeb noticed that the key feature of the squared norm function that is relevant to the Kempf-Ness theorem is its plurisubharmonicity, yielding the following 
\begin{theorem}[Azad-Loeb~\cite{Azad-Loeb}]
\label{thm:azad}
Assume given a complex reductive group $\Ggroup$ and a maximal compact subgroup $\Kgroup$. Let $\Hgroup$ be a closed complex subgroup of $\Ggroup$ and $\varphi:\Ggroup/\Hgroup \to \C$ a strictly plurisubharmonic function. If the critical point set of $\varphi$ is non-empty then it is a $\Kgroup$-orbit and $\varphi$ achieves its global minimum on this orbit.
\end{theorem}
This theorem allows us to expand the class of minimization principles as in Theorem~\ref{thm:intro}. The main idea is that plurisubharmonic functions may be easily constructed from holomorphic functions using convexity. For example, if $f: \Md^N(\C) \to \C$ is holomorphic, then $\log|f|$ is plurisubharmonic. Similarly, any norm on $\Md(\C)$ is plurisubharmonic. In particular, since we may define a norm on $\Md^N(\C)$ by summing over the Schatten $p$-norms
\begin{equation}
    \label{eq:schatten}
    \|\ww\|_{p} := \sum_{k=1}^N \|W_k \|_p, 
\end{equation}
we obtain a strictly plurisubharmonic function on $\Md^N(\C)$, and thus by restriction, strictly plurisubharmonic functions on $\groupx$ when $1<p<\infty$. 
Theorem~\ref{thm:azad} then implies the following general regularization principle.
\begin{theorem}
\label{thm:intro-revised} Assume $X$ has full rank and $1<p<\infty$. Then
\begin{equation}
\label{eq:variation2} \argmin_{\ww \in \fiberX} \|\ww\|_p= \fiberX \cap \balance.
\end{equation}   
\end{theorem}
These generalizations are not entirely satisfactory. In practice, it is the $L^2$ (ridge) and $L^1$ (lasso) regularization that matter the most. While the function $\|\ww\|_1$ is plurisubharmonic  on $\groupx$, it is not {\em strictly\/} plurisubharmonic. This leaves open interesting possibilities; for example, the set or critical points for the $L^1$-regularizer may not be a $U_d^{N-1}$-orbit. It is also of interest to study the related gradient flows.

\section{Discussion: a new dynamic paradigm}
Our results provide a paradigm for training dynamics illustrated in Figure~\ref{fig:reg-flow} and Figure ~\ref{fig:learn-flow}. In this idealization, we consider regularization and learning as two distinct dynamic processes. Training is assumed to take place in two stages. First, a fast regularization provides the optimal parameter description of the training data (Figure~\ref{fig:reg-flow}). This is then followed by a slower learning stage, in which the parametric representation minimizes the cost function, while staying optimal at all times (Figure~\ref{fig:learn-flow}). 

This decomposition offers a conceptual framework for training dynamics that is based on the intrinsic geometry of parameter space induced by the neural architecture. It is also amenable to a rigorous analysis within the dynamical systems framework for fast-slow systems, since both the learning and regularization flow admit several explicit descriptions (see~\cite{Chen2} for solutions to the learning flow). In the framework for fast-slow analysis these idealized flows should be seen as limiting descriptions of training dynamics arising from the following models. 

\subsection{Gradient flow of a regularized cost function}
This is the gradient flow on $\Md^N$
\begin{equation}
    \label{eq:fast-slow1}
    \dot{\ww} = -\nabla_\ww \left( E(X(\ww)) + \frac{\kappa}{2} \|\ww\|_2^2\right),
\end{equation}
where the parameter $\kappa>0$ controls the strength of the regularization. The main observation then is that this dynamical system may be naturally decomposed  at each point $\ww$ into two orthogonal flows, one normal to $\fiberX$ (learning) and the other parallel to $\fiberX$ (regularizing). Our heuristic idea is that regularization is `fast' because of the exponential rate of convergence provided by Theorem~\ref{thm:reg-flow} (note that the rate is now $2\kappa$, not $2$), so that the dynamics of equation~\eqref{eq:fast-slow1} may be rigorously approximated by the learning and regularizing flows.

It is of interest to formalize the heuristic of fast-regularization and slow-learning for equation~\eqref{eq:fast-slow1} using the geometric singular perturbation theory of Fenichel~\cite{Fenichel}. 


\subsection{Regularization by background noise} \label{subsec:noise}
It is important to note that small noise naturally provides $L^2$ regularization as follows.

Fix an inverse temperature $\beta \in (0,\infty)$ and let $\bfB_t$ denote the standard Brownian motion in $\Md^N$. A natural model for background noise in the parameter space $\Md^N$ is the Ornstein-Uhlenbeck process described by the Langevin equation 
\begin{equation}
    \label{eq:fast-slow2}
    d\ww_t = -\kappa \ww_t + \sqrt{\frac{2}{\beta}} d\bfB_t.
\end{equation}
The equilibrium measure for $\ww_t$ is the Gaussian with probability density
\begin{equation}
    \label{eq:fast-slow3}
    \rho_{\beta,\kappa} (\ww) = \frac{1}{Z_{\beta,\kappa}} e^{-\beta\kappa \|\ww\|^2}, \quad Z_{\beta,\kappa} = \int_{\Md^N} e^{-\beta\kappa \|\ww\|^2} d\ww.
\end{equation}
We allow ourselves two parameters $(\beta,\kappa)$ to independently study the effect of the small noise ($\beta \to \infty$) and small regularizer ($\kappa\to 0)$ limits. However, it is only the product $\beta\kappa$ that determines the above density. 

The background noise may be naturally included in training dynamics by studying the Langevin equation
\begin{equation}
    \label{eq:fast-slow4}
    d{\ww}_t = -\left( \nabla_\ww \left( E(X(\ww_t)\right)+ \kappa \ww_t \right)\,dt  + \sqrt{\frac{2}{\beta}}\, d\bfB_t.
\end{equation}
The noise in this equation is isotropic. However, one may also consider anisotropic stochastic forcing that corresponds to the idealized gradient flows for regularization 
and learning. These are Riemannian Langevin equations (RLE),
where the stochastic forcing corresponds to Brownian motion at inverse temperature $\beta$ on the manifolds $(\balance,\iota)$ and $(\fiberX,\iota)$. The explicit description of these equations in coordinates is quite subtle since it includes deterministic corrections by curvature. We present an analysis of this effect on $\balance$ in~\cite{MY-RLE}.  We note that geometric singular perturbation theory for noisy fast-slow systems has been recently introduced~\cite{Kuehn}. 

\subsection{Is deep learning `secretly Bayesian'?}
Theorems~\ref{thm:intro}-\ref{thm:reg-flow} along with these RLE suggests a Bayesian interpretation for deep learning. This goes as follows.

Assume that $\Md^N$ is equipped with the Gaussian prior in equation~\eqref{eq:fast-slow3}. Now condition on the end-to-end matrix $X$; the posterior measure is  Gaussian measure restricted to $(\fiberX,\iota)$ yielding the partitition function
\begin{equation}
    \label{eq:fast-slow5}
     \tilde{Z}_{\beta,\kappa} = \int_{\fiberX} e^{-\beta\kappa \|\ww\|^2} d\mathcal{H}^{(N-1)d^2)}(d\ww).
\end{equation}
Here $\mathcal{H}^{(N-1)d^2)}(d\ww)$ is the volume element obtained by restricting Lebesgue measure on $\Md^N$ to $\fiberX$. Theorem~\ref{thm:intro} then immediately implies that when the noise is small ($\beta \to \infty)$ the posterior measure is the uniform measure on $\orbitx$. In this limit, the microscopic dynamics are described by Brownian motion on $\orbitx$, which is constructed explicitly in~\cite{MY-RLE}. We may also change variables using Lemma~\ref{le:group-orbit} to rewrite $\tilde{Z}_{\beta,\kappa}$ as an integral over $\GLdC^{N-1}$ which is amendable to evaluation using representation theory (see~\cite{McSwiggen} for an introduction to similar integrals). 

Both these approaches are studied in forthcoming work. While Bayesian principles in this form can only be made mathematically precise for the DLN at this time, our work is broadly inspired by the goal of developing rigorous geometric foundations for deep learning in the spirit of~\cite{Belkin-Niyogi,LeCun}. Our work to date in these directions includes a geometric decomposition of the tangent space for ReLU networks by the first author~\cite{GrigsbyLindseyPersistentPseudodimension} and a re-investigation of the Nash embedding theorems by Inauen and the second author~\cite{IM}.

\subsection{Conclusion}
 These questions reveals the power of the DLN as a phenomenological model for deep learning. While the DLN is amenable to the tools of dynamical system theory and stochastic differential geometry, each such study requires a careful geometric analysis, and seems to reveal new connections between training dynamics as studied in practice and the underlying mathematical foundations.


\section{Acknowledgements}
The authors express their gratitude to Yotam Alexander, Sanjeev Arora, Nadav Cohen, Boris Hanin, Colin McSwiggen, Linda Ness, Noam Razin, Karen Uhlenbeck and Tianmin Yu for stimulating conversations regarding this work.

\bibliographystyle{siam}
\bibliography{dln-balancedness}

\end{document}